\documentclass[letterpaper, 10 pt, conference]{ieeeconf}  %

\IEEEoverridecommandlockouts                              %

\usepackage[utf8]{inputenc}
\usepackage[english]{babel}
\usepackage{graphicx}
\usepackage{amsmath,amssymb}
\usepackage{amsfonts}
\usepackage{upgreek}
\usepackage{color}
\usepackage{marginnote}
\usepackage[ruled,vlined]{algorithm2e}
\usepackage{multirow}
\usepackage{paralist}
\usepackage{wrapfig}
\usepackage{xcolor}
\usepackage{hyperref}

\let\labelindent\relax
\usepackage{amsthm}
\usepackage{environ}
\usepackage[normalem]{ulem}
\usepackage{enumitem}
\usepackage{subcaption}
\usepackage{url}
\usepackage{hyperref}

\newtheorem{theorem}{Theorem}

\newtheorem*{problem}{Problem}

\newtheorem{definition}{Definition}

\usepackage{fancyvrb}
\usepackage{listings}
\usepackage{scalerel,stackengine}
\stackMath

\def\C{{\cal C}} %
\def\D{{\cal D}} %
\def\H{{\cal H}} %
\def\K{{\cal K}} %
\def\L{{\cal L}} %
\def\S{{\cal S}} %
\def\U{{\cal U}} %
\def\X{{\cal X}} %

\newcommand{\nnnum}[1]{\relax\ifmmode 
	{\mathbb #1}_{\geq 0} \else ${\mathbb #1}_{\geq 0}$
	\fi}
\newcommand{\npnum}[1]{\relax\ifmmode 
	{\mathbb #1}_{\leq 0} \else ${\mathbb #1}_{\leq 0}$
	\fi}
\newcommand{\pnum}[1]{\relax\ifmmode 
	{\mathbb #1}_{> 0} \else ${\mathbb #1}_{> 0}$
	\fi}
\newcommand{\nnum}[1]{\relax\ifmmode 
	{\mathbb #1}_{< 0} \else ${\mathbb #1}_{< 0}$
	\fi}
\newcommand{\plnum}[1]{\relax\ifmmode 
	{\mathbb #1}_{+} \else ${\mathbb #1}_{+}$
	\fi}
\newcommand{\nenum}[1]{\relax\ifmmode 
	{\mathbb #1}_{-} \else ${\mathbb #1}_{-}$
	\fi}

\newcommand{\reals}{{\mathbb{R}}}                    %

\newcommand\reallywidehat[1]{%
\savestack{\tmpbox}{\stretchto{%
  \scaleto{%
    \scalerel*[\widthof{\ensuremath{#1}}]{\kern-.6pt\bigwedge\kern-.6pt}%
    {\rule[-\textheight/2]{1ex}{\textheight}}%
  }{\textheight}%
}{0.5ex}}%
\stackon[1pt]{#1}{\tmpbox}%
}

\newcommand{\RN}[1]{%
  \textup{\uppercase\expandafter{\romannumeral#1}}%
}

\NewEnviron{smallalign}{%
\scalebox{0.8}{
\begin{align*}
\BODY
\end{align*}}}

\newcommand{\rp}{\reals^{\geq 0}}
\newcommand{\se}{s}
\newcommand{\cm}{c}
\newcommand{\gp}{g}
\newcommand{\co}{\pi}
\newcommand{\T}{\intercal}
\newcommand{\unif}[1]{\mathtt{Unif}\left(#1\right)}
\newcommand{\relu}{\mathtt{ReLU}}
\newcommand{\indicator}{\mathbb{I}}
\newcommand{\E}{\mathop{\mathbb{E}}}

\newcommand{\iv}{r}
\newcommand{\hr}[1]{\mathcal{E}\left(#1\right)}
\newcommand{\piv}{\se^+}
\newcommand{\diag}{\mathop{\mathrm{diag}}}
\newcommand{\ce}{\mathrm{center}}

\title{\LARGE \bf
Learning Certifiably Robust Controllers Using Fragile Perception}

\author{Dawei Sun$^{1}$, Negin Musavi$^{1}$, Geir Dullerud$^{1}$, Sanjay Shakkottai$^{2}$, and Sayan Mitra$^{1}$%
\thanks{$^{1}$Dawei Sun, Negin Musavi, Geir Dullerud, and Sayan Mitra are with Coordinated Science Laboratory,
        University of Illinois, Urbana, IL 61801, USA
        {\tt\small \{daweis2, nmusavi2, dullerud, mitras\}@illinois.edu}}%
\thanks{$^{2}$Sanjay Shakkottai is with the Department of Electrical and Computer Engineering, University of Texas at Austin,,
        Austin, TX 78712, USA
        {\tt\small sanjay.shakkottai@utexas.edu}}%
}

\begin{document}

\maketitle
\thispagestyle{empty}
\pagestyle{empty}

\begin{abstract}
Advances in computer vision and machine learning enable robots to perceive their surroundings in powerful new ways, but  these perception modules have well-known  fragilities. We consider the problem of synthesizing a safe controller that is robust despite perception errors. The proposed method constructs  a state estimator based on Gaussian processes with input-dependent noises. This estimator  computes a high-confidence set for the actual state given a perceived state. Then, a robust neural network controller is synthesized that can provably handle the state uncertainty. Furthermore, an adaptive sampling algorithm is proposed to jointly improve the estimator and controller. Simulation experiments, including a realistic vision-based lane keeping example in CARLA, illustrate the promise of the  proposed approach in synthesizing  robust controllers with deep-learning-based perception.
\end{abstract}

\section{Introduction}
\label{sec:intro}

Advances in computer vision and machine learning enable robots to perceive their surroundings in powerful new ways, but  these perception modules have well-known  fragilities.
The decision boundaries for classifiers are vulnerable to adversarial inputs~\cite{DBLP:journals/corr/SzegedyZSBEGF13}. 
ML models are often overconfident, and do not know what they do not know~\cite{gal2016dropout}.
Biases in training data can bleed into biased algorithms for pedestrian detection that do not work well for dark-skinned people~\cite{wilson2019predictive}.
Embodied in robots or autonomous vehicles, like driver-assistance systems, robotic tractors, and delivery drones, the perception module fragilities can trigger safety violations. 

We address the problem of designing robust controllers that tolerate and compensate for the fragilities of the perception modules they use. Simply put, {\em how to design reliable controllers that use unreliable perception?}  We study the fundamental control task for a robot to maintain a given {\em invariant\/}. For example, a drone has to stay within a geo-fenced area, or a car has to stay within the lanes. The system uses computer vision-based perception. There are several technical challenges in rigorously addressing this problem.
First, unlike the typical i.i.d. model of sensor noise in control theory textbooks, the output from the perception modules here may have wildly biased, state-dependent errors. Secondly, leaving aside the issue of perception errors, even the task of  designing an invariance-preserving controller for a nonlinear, (and possibly incompletely known) dynamical system is an actively researched topic~\cite{ames2019control,fan2020fast,fan2018controller,sun2020learning,Vasudevan-RSS-19,herbert2017fastrack}. Finally, certifying the correctness or safety of a given control system is also another challenge actively pursued by the formal verification and the control theory communities.

In this paper, the perception module is modeled as a function that takes in the actual state and generates a noisy copy of it, i.e., the perceived state. In order to design a feedback controller that only has access to the perceived state to maintain safety, we utilize Gaussian processes (GP) with input-dependent noises to construct a set-valued state estimator from a data set of input-output pairs of the perception module, such that given a perceived state it computes a high-confidence set for the actual state. We then design a learning algorithm for synthesizing a controller that takes in the high-confidence set from the state estimator and computes the control signal. A barrier function is jointly synthesized with the controller, which provides theoretical guarantee on the safety of the resulting closed-loop system. Inspired by~\cite{dean2020guaranteeing}, the original theory of control barrier functions is extended to handle uncertainty in the state. However, the synthesis could fail due to the large uncertainty of the GP-based state estimator. In this case, new samples from the perception module will be collected to improve the state estimator until a satisfying controller and certificate are found. An adaptive sampling algorithm is proposed to reduce the number of samples required in this process.

We evaluated the proposed approach on three benchmark systems including two simpler ones with synthetic perception errors and one more realistic one in the CARLA simulator with a deep-learning-based perception module. Experimental results clearly verified that with the proposed approach, one can synthesize controllers that are robust to perception errors, and the proposed adaptive sampling algorithm indeed improves the sample efficiency.

Our main contributions are threefold.
(1) We formulate the problem of state estimation with uncertainty bounds as a GP regression problem with input-dependent noises, and show that indeed, this formulation can be effective for realistic deep-learning-based perception modules. 
(2) We design a learning-based algorithm that models the controller and the certificate as neural networks, and show how state estimation uncertainty bounds can be used for synthesizing controllers that are certifiably robust to perception fragilities. (3) We propose an adaptive sampling approach to reduce the number of samples required to learn such a certifiable controller.

\section{Related work}
\label{sec:related}
\noindent\textbf{Perception-based control.} As new types of sensors emerge, the problem of integrating these sensors and the corresponding perception module into the control pipeline has attracted interest.
In recent works~\cite{lin2018autonomous,loianno2016estimation, tang2018aggressive}, the perception-based controller has been studied to enable aggressive control for drones. Further, data-driven approaches have been developed by the machine learning community. For example, imitation learning~\cite{codevilla2018end} and reinforcement learning~\cite{Sadeghi-RSS-17} have been used to learn vision-based control policies. However, such purely data-driven approaches do not provide any safety guarantees. In a series of recent works by Dean et al., the authors studied the robustness guarantees of perception-based control algorithms. In~\cite{dean2020robust}, the authors proposed a perception-based controller synthesis approach for linear systems and provided a theoretical analysis.
In~\cite{dean2020guaranteeing}, the authors proposed robust barrier functions which provide an approach for synthesizing safety-critical controllers under uncertainty of the state. In our approach, such a robust barrier function is used as a component of the estimation and control pipeline.

\noindent\textbf{Analysis of systems with ML-based modules.} As machine learning becomes the dominant approach in autonomy, analysis of such ML-based systems attracts more and more interest.
VerifAI~\cite{dreossi2019verifai} provides a complete framework for analyzing autonomous systems with ML modules in the loop.
In~\cite{hsieh2021verifying}, the authors study the safe abstraction of systems with ML-based perception modules. In~\cite{everett2021efficient}, the reachability of a closed-loop system with a neural controller has been studied. In~\cite{sun2022neureach}, a data-driven reachability analysis tool is developed, which works for ML-based systems.

\noindent\textbf{Certifiable controller synthesis.}
Recently, the idea of jointly synthesizing certificates and controllers with machine learning has been popular. Unlike purely data-driven approaches, such approaches can draw guarantees on the performance of the synthesized controller. For example, control Lyapunov functions~\cite{chang2019neural}, contraction metrics~\cite{sun2020learning}, and barrier functions~\cite{qin2021learning} have all been studied under this setting.

\section{Preliminaries and problem setup}
\label{sec:prelim}
\begin{figure*}
    \centering
    \includegraphics[width=\linewidth]{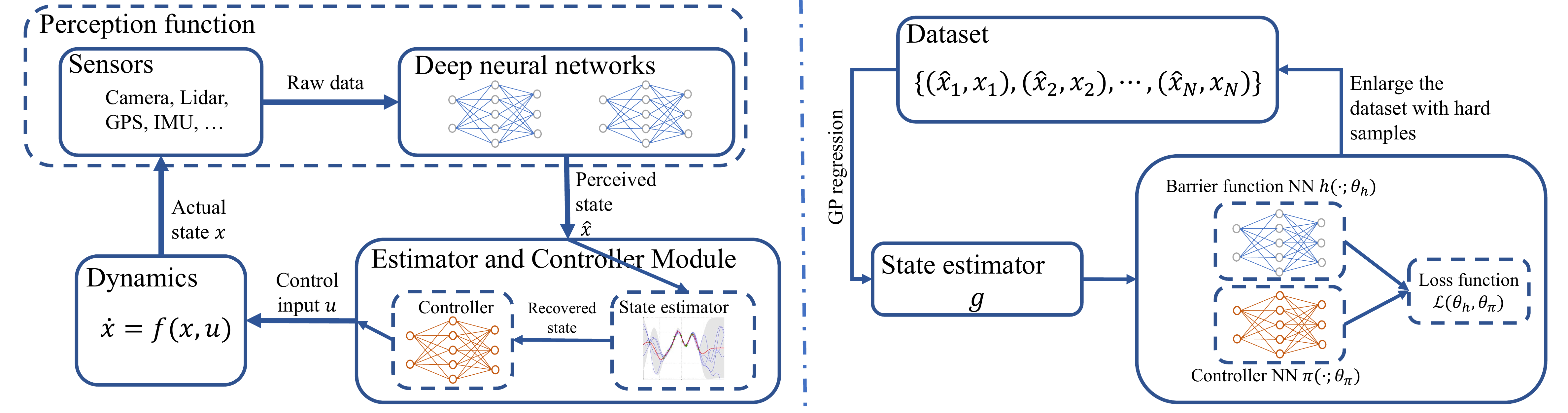}
    \caption{Overview of the system and the learning algorithm.}
    \label{fig:overall_sys}
\end{figure*}
We denote by $\reals$ and $\rp$ the set of real and non-negative real numbers respectively. Unless otherwise stated, $x^{(i)}$ denotes the $i$-th element of vector $x$. Similarly, for a vector-valued function $f : \reals^n \mapsto \reals^m$, $f^{(i)} : \reals^n \mapsto \reals$ is its $i$-th component.
For a vector $x \in \reals^n$, $\diag(x)$ is the $n \times n$ diagonal matrix generated from $x$. Also, let $\exp(x) = [\exp(x^{(i)}), \cdots, \exp(x^{(n)})]$.
For a point $c \in \reals^n$ and a positive definite matrix $Q \in \reals^{n \times n}$, denote $\hr{c, Q} := \{x \in \reals^n : x^\T Q x \leq 1\}$ to be the ellipsoid at $c$ with $Q$ defining its shape. Further, let $\ce(A)$ be the center of an ellipsoid $A$.
For a set $\S$, we denote $\unif{\S}$ the uniform distribution over $\S$, $2^\S$ the power set of $\S$, and $\indicator_\S$ the indicator function generated from this set, which is a binary-valued function such that $\indicator_\S(s) = 1$ iff. $s \in \S$.
Let $\relu$ be the Rectified Linear Unit (ReLU) function. That is, $\relu(x) = x$ if $x>0$, otherwise, $\relu(x) = 0$.
An extended class $\K_\infty$ function is a continuous function $\alpha: \reals \mapsto \reals$ that is strictly increasing with $\alpha(0) = 0$ and $\lim_{r \rightarrow \infty}\alpha(r) = \infty$.

\subsection{Dynamical systems}
We consider dynamical systems of the form
\begin{equation}
    \dot{x} = f(x, u),
    \label{eq:sys}
\end{equation}
where $x \in \X \subseteq \reals^n$ is the state and $u \in \mathcal{U} \subseteq \reals^m$ is the control input. Here, $\X$ and $\mathcal{U}$ are the state space and the control input space, which are compact sets. We assume that $f: \reals^n \times \reals^m \mapsto \reals^n$ is a smooth function.
A feedback controller is a function $\co: \X \mapsto \U$, which can be plugged into the dynamical system in Eq.~\eqref{eq:sys} and results in an autonomous system $\dot{x} = f(x, \co(x))$.
A function $\xi : \X \times \rp \mapsto \X$ is called the \textit{solution function} of the autonomous system if for all initial states $x_0 \in \X$, $\xi(x_0, \cdot)$ satisfies Eq.~\eqref{eq:sys}. Given an initial state $x_0$, the function of time $\xi(x_0, \cdot)$ is called a trajectory of the system.

In this paper, we only assume input-output access to the function $f$, i.e., $f(x, u)$ can be computed given the state $x$ and control input $u$. This enables the application in practical situations where a closed-form description of $f$ is unavailable or unwieldy. For example, $f$ could  be a deep neural network model of a complicated robot.

\subsection{Imperfect perception functions}
We study autonomous systems equipped with sensors and perception modules, which together constitute a \textit{perception function} $\se : \X \mapsto \X$.
That is, it takes in the actual state $x \in \X$ and produces a noisy observation $\hat{x} := \se(x)$ which is called the \textit{perceived state}.
The perception function $\se$ is a complex composite of the environment, the sensor, and the perception module, for instance, a deep convolutional neural network, and we treat it as a black box function. We need not have the closed-form expression of the perception function, instead we only need input-output access to it, i.e., given a state $x$, it is possible to run the perception function and compute the perceived state $\hat{x}$. However, this evaluation process is expensive and should be called as few times as possible. In this paper, we \textit{do not} assume the smoothness or invertibility of function $\se$. The perception function $\se$ could map two states to the same perceived state, in which case, recovering the exact actual state from the perceived state becomes impossible, and the inverse of the perception function $\se$ should be a set-valued function.

\subsection{The synthesis problem}
We first introduce the notion of invariant sets.

\begin{definition}[Invariant sets]
\label{def:inv}
Consider an autonomous system $\dot{x} = f(x)$, where $x \in \X$. 
A set $\C \subseteq \X$ is called an invariant set of the autonomous system, if starting from any initial state in $\C$, a trajectory always stays in $\C$, i.e.,
$x_0 \in \C \implies \forall t \in \rp, \xi(x_0, t) \in \C$.
\end{definition}
The notion of invariant sets is widely used in safety-critical synthesis. For example, if one can ensure that the user-defined safe set is invariant to the system, then starting from a safe state, the system will always maintain safety.

The synthesis problem considered in this paper is as shown in Figure~\ref{fig:overall_sys}. The goal is to synthesize a module that computes control input $u$ for the dynamical system such that after plugging this module into the dynamical system, a user-defined set $\S \subset \X$ is invariant to the closed-loop system. However, different from an ordinary feedback controller, this module does not have access to the actual state $x$ of the system. Instead, it only has access to a noisy version of the actual state, which is the perceived state $\hat{x}$. Formally, we have the following problem definition.

\begin{problem}[Synthesis problem]
Given a dynamical system as in Eq.~\eqref{eq:sys}, a perception function $\se: \X \mapsto \X$, and a target invariant set $\S \subset \X$, the synthesis problem is to find a function $\cm: \X \mapsto \U$, such that $\S$ is indeed invariant with respect to the closed-loop system, i.e.,
$$\dot{x} = f(x, \cm(s(x))).$$
\end{problem}

Our proposed approach to solving the above problem uses Gaussian processes and certificate-based robust controller synthesis. A GP-based state estimator is constructed such that given a perceived state $\hat{x}$ it computes a high-confidence set that contains the actual state $x$ with a high probability.
Then, based on the GP-based state estimator, the approach jointly searches for a certificate (a control barrier function) and a controller. Input-output pairs of the perception function are adaptively sampled to improve the GP-based state estimator until a certificate and a certified controller are found. With the certificate, we can theoretically show the safety of the closed-loop system and its robustness to perception error. The approach could fail to find such a controller in cases where the perception function maps multiple vastly different states to the same perceived state and the controller cannot obtain enough information from the perceived state. In this case, the algorithm outputs a list of perceived states that prevents a successful controller synthesis. This information can be useful for the user to improve the perception function.

\section{Overview of Design Methodology}
\label{sec:approach}
Our approach decomposes the desired controller function $\cm$ into two components, a GP-based \textit{state estimator} $\gp: \X \mapsto 2^\X$ aiming to recover a high-confidence set for the actual state from the perceived state, and a robust controller $\co:\X \mapsto \U$ (See Fig.~\ref{fig:overall_sys}). As will be shown later, the output of $\gp$ is an ellipsoid. We call the sequential combination of the two components as the {\em Estimator and Controller Module (ECM)\/}: $\cm(\cdot) := \co(\ce(\gp(\cdot)))$.

This section is organized as follows. In Section~\ref{sec:gp}, we will show a GP-based algorithm for constructing the state estimator $\gp$.
In Section~\ref{sec:cbf}, we will show a learning-based algorithm that learns a robust controller $\co$ given the state estimator $\gp$. In Section~\ref{sec:rtas}, we will combine these two components and show an algorithm that jointly improves the state estimator and the controller with adaptive sampling.

\subsection{Constructing the state estimator}
\label{sec:gp}
The ECM consists of a state estimator $\gp$ that estimates the ground-truth state $x$ from the perceived state $\hat{x} = \se(x)$ (Recall Fig.~\ref{fig:overall_sys}).
Since $\se$ might not be exactly invertible, there might be a set of possible states that correspond to a perceived state $\hat{x}$. This suggests designing the state estimator $\gp$ as a set-valued function. Further, the state estimator $\gp$ is constructed from a finite set of samples, we can only expect the estimate $g(\hat{x})$ will contain the actual state with a high probability. Thus, the state estimator is designed as $\gp: \X \mapsto 2^\X$ such that for an arbitrary state $x$, $x \in \gp(\se(x))$ with a high probability.
In this paper, we construct such an estimator using Gaussian processes.
Gaussian processes (GP) are widely used non-parametric models for regression. Here, we view the problem of estimating $x$ from $\hat{x}$ as a regression problem. One of the key features of GP regression is that given a query $\hat{x}$, it does not compute a single-point estimate of $x$ but a posterior distribution of $x$, from which a high-confidence set containing $x$ can be extracted.

\vspace{1mm}
\noindent\textbf{Construction of the data set.}
The state estimator will be constructed from samples of the perception function $\se$. To this end, a data set $\D = \{(\hat{x}_j, x_j)\}_{j=1}^{N}$ that captures the relationship between the perceived state and the actual state is constructed. Each sample is obtained as follows: first, a state $x_j \in \X$ is sampled, then $\hat{x}_j$ is computed as $\hat{x}_j = \se(x_j)$.

\vspace{1mm}
\noindent\textbf{Setup of the probabilistic model.}
In order to apply GP regression, we first set up a probabilistic model to represent the observations in the above data set. As stated earlier, since $\se$ might not be invertible, there might be multiple $x$'s that correspond to the same $\hat{x}$ in the above data set. To characterize this property of the data set, we assume there is an input-dependent observation noise. Specifically, we adopt the following probabilistic model (Along the lines of~\cite{kersting2007most}).
$$x_j = \piv(\hat{x}_j) + w_j,\, j = 1, \cdots, N,$$
where $\piv : \X \mapsto \X$ approximately inverts $\se$ with an input-dependent zero-mean noise $w_j \sim \mathcal{N}\left(0, \diag(\exp(z(\hat{x}_j)))^2\right)$. Here, $z : \X \mapsto \reals^n$ is the noise-level function, which characterizes the non-invertibility of the perception function $\se$ at $\hat{x}$.
The goal is to learn the function $z$ and the ``inverse" function $\piv$ from data. Notice that the perceived state $\hat{x}$ is a noisy copy of the actual state $x$. Given a perceived state $\hat{x}$, the actual state $x$ should be close to $\hat{x}$. Thus, instead of estimating $x$ directly, it should be easier to estimate an error $e$ such that $x = \hat{x} + e$. Here, the error $e$ is called the \textit{perception error}. To this end, we decompose $\piv$ as the summation of an identity mapping and a function $\iv : \X \mapsto \X$ as follows.
$$\piv(a) = a + \iv(a),\,\forall a \in \X.$$
Now, we can rewrite our probabilistic model as the relationship between the perception error and the perceived state,
$$e_j := x_j - \hat{x}_j = r(\hat{x}_j) + w_j,\, j = 1, \cdots, N,$$

Next, we will use GP regression to identify the functions $\iv$ and $z$. Instead of identifying the vector-valued functions directly, we assume the components of these functions are independent and thus apply GP regression to identify each component individually. Here, we only show the process for the $i$-th components, i.e., $\iv^{(i)}$ and $z^{(i)}$. As in common settings in GP literature, for example \cite[Chapter~2.2]{rasmussen2003gaussian}, we first put priors on functions $\iv^{(i)}$ and $z^{(i)}$. Specifically, we assume that $\iv^{(i)}$ is a sample from a zero-mean Gaussian process generated by a kernel $k_{\iv^{(i)}}: \X \times \X \mapsto \rp$ or formally $\iv^{(i)} \sim \mathcal{GP}\left(0, k_{\iv^{(i)}}(\cdot, \cdot)\right)$. That is, we assume that the function $\iv^{(i)}$ is realizable by a Gaussian process. As for the unknown function $z^{(i)}$, we follow the approach in~\cite{kersting2007most} and assume that $z^{(i)}$ is also a sample from a Gaussian process, $z^{(i)} \sim \mathcal{GP}\left(0, k_{z^{(i)}}(\cdot, \cdot)\right)$, generated by a kernel $k_{z^{(i)}}$.

\vspace{1mm}
\noindent\textbf{Computation of the posterior distribution.}
The prior distributions of $\iv^{(i)}$ and $z^{(i)}$ are the best estimate we can make before we sample from the underlying function. However, with the data, this estimate can be improved by computing the \textit{posterior distribution}. Specifically, the goal is, given $\D$ and a query $\hat{x}_*$, to compute the posterior distribution of $e_*^{(i)}$, i.e., $e_*^{(i)} | \hat{x}_*, \D$. We adopt the approach in~\cite{kersting2007most}, which proceeds as follows. Given data set $\D$, we first estimate a standard, homoscedastic\footnote{That is, the noise term is independent of the input.} GP. Then, this GP is evaluated on $\D$, and the empirical noises are recorded and form a new data set to compute the posterior of the noise-level function $z^{(i)}$. For more details of the computation, please refer to~\cite[Section 4]{kersting2007most}. Let the posterior of $z^{(i)}$ be $\hat{z}^{(i)}$.

Then, it is a standard result (e.g., please see \cite{goldberg1997regression}) that given the data set $\D$ and a query $\hat{x}_*$, the posterior distribution of $e_*^{(i)}$ is still Gaussian, $e_*^{(i)} | \hat{x}_*, \D \sim \mathcal{N}\left(\mu^{(i)}(\hat{x}_*), (\sigma^{(i)}(\hat{x}_*))^2\right)$, where the mean function $\mu^{(i)}$ and standard deviation function $\sigma^{(i)}$ are defined as follows.
\[
\mu^{(i)}(\hat{x}_*) = K(\hat{x}_*, \hat{X}) (K(\hat{X}, \hat{X}) + K_N)^{-1} E,
\vspace{-0.3cm}
\]
\begin{multline*}
\sigma^{(i)}(\hat{x}_*) = k_{\iv^{(i)}}(\hat{x}_*, \hat{x}_*) + \exp(\hat{z}^{(i)}(\hat{x}_*)) \\- K(\hat{x}_*, \hat{X}) (K(\hat{X}, \hat{X}) + K_N)^{-1} K(\hat{x}_*, \hat{X}),
\end{multline*}
where $\hat{X} = [\hat{x}_1, \cdots, \hat{x}_N]^\T$, $E = [e^{(i)}_1, \cdots, e^{(i)}_N]^\T$, $K_N = \diag([\hat{z}^{(i)}(\hat{x}_1), \cdots, \hat{z}^{(i)}(\hat{x}_N)])$. The matrices are defined as follows.
\[
K(\hat{x}_*, \hat{X}) \in \reals^{1 \times N}, ~ K(\hat{x}, \hat{X})_{j} = k_{\iv^{(i)}}(\hat{x}_*, \hat{x}_j),
\]
\[
K(\hat{X}, \hat{X}) \in \reals^{N \times N}, ~ K(\hat{X}, \hat{X})_{jl} = k_{\iv^{(i)}}(\hat{x}_j, \hat{x}_l).
\]
It should be clear that the posterior distribution of the actual state $x_* = \hat{x}_* + e_*$ is also Gaussian, $x_* | \hat{x}_*, \D \sim \mathcal{N}\left(\hat{x}_* + \mu(\hat{x}_*), \diag(\sigma(\hat{x}_*))^2\right)$.

\vspace{1mm}
\noindent\textbf{Construction of the high-confidence set.}
Next, we construct a high-confidence set such that the actual state is contained in it, with high probability. Since the posterior distribution of the actual state $x$ is a multivariate Gaussian distribution, it is standard to construct an ellipsoid as the high-confidence set.
Specifically, given a confidence level $\delta$, we construct the high-confidence set $\gp(\hat{x})$ as an ellipsoid centered at the mean of the distribution, $\hat{x} + \mu(\hat{x})$, and whose semiaxes are proportional to the standard deviations $\{\sigma^{(i)}\}$ such that $\Pr\left(x \in \gp(\hat{x})\right) = \delta$. The computation of this ellipsoid follows from standard results in statistics.

\label{sec:gp}
\subsection{Learning the controller}
\label{sec:cbf}
In this section, we will elaborate on the process of learning a certified controller $\co$ given such a state estimator $\gp$.
Throughout this section, we fix the \textit{target} invariant set $\S \subset \X$ for the control system. In order to prove or certify that $\S$ is indeed invariant with respect to the closed-loop system, we will learn a continuously differentiable function $h: \X \mapsto \reals$, the {\em certificate}, such that 
the $0$-superlevel set of $h$, 
$\C_h := \{x \in \X : h(x) > 0\}$ is equal to $\S$. Functions like $h$ are often called barrier functions or barrier certificates.

Barrier certificates were introduced  by Prajna and Jadbababie in~\cite{prajna2004safety}  formalizing the folk theorem that at the boundary of an invariant set, the vector field of an autonomous system must be pointing inwards. The approach was used  earlier to prove invariance of complex hybrid systems (see, for example~\cite{MWLF-hscc03}). A natural extension of this idea to controlled systems is presented in~\cite{ames2019control}, which we quote here:

\begin{theorem}
 If  $\co: \X \mapsto \U$ be a differentiable controller such that there  exists  an extended class $\K_\infty$ function  $\alpha$ such that $\forall x \in \X$,
\begin{equation}
    \frac{\partial h}{\partial x} \cdot f(x, \co(x)) + \alpha(h(x)) \geq 0,
\end{equation}
then, $\C_h$ is invariant to $\dot{x} = f(x, \co(x))$.
\label{thm:CBF}
\end{theorem}
The function $h$ is called a \textit{control barrier function}.
All of the above, assume that the controller $\co$ has access to the perfect state. Going forward, we relax this assumption. Instead, we assume that $\co$ only has access to the output {\em set\/} from the estimator $\gp$, and the actual state can be anywhere in that set. 
Control barrier functions can be extended to handle uncertainty in states as in~\cite{dean2020guaranteeing}.
Specifically, we have the following theorem.
\begin{theorem}
Assume that a state estimator $\gp$ satisfies that for all $x$, the set $\gp(\se(x))$ contains $x$. If  $\co: \X \mapsto \U$ be a differentiable controller such that there exists an extended class $\K_\infty$ function $\alpha$ such that $\forall \hat{x} \in \X$,
\begin{equation}
    \inf_{x \in \gp(\hat{x})} \left( \frac{\partial h}{\partial x}(x) \cdot f(x, \cm(\hat{x})) + \alpha(h(x)) \right) \geq 0,
\end{equation}
where $\cm(\cdot) := \co(\ce(\gp(\cdot)))$, then, $\C_h$ is invariant to the closed-loop system, i.e., $\dot{x} = f(x, \cm \circ \se(x))$.

\label{thm:rCBF}
\end{theorem}

\begin{proof}
Let us denote the resulting feedback controller by $\hat{\co}$, i.e., $\hat{\co} := \cm \circ \se$. Then, we show that $\hat{\co}$ indeed satisfies the conditions in Theorem~\ref{thm:CBF}. For an arbitrary $x \in \X$, we know that $x \subset g(\se(x)) = g(\hat{x})$, and thus by assumption, we have that
\begin{equation}
    \frac{\partial h}{\partial x}(x) \cdot f(x, \hat{\co}(x))) + \alpha(h(x)) \geq 0.
\end{equation}
Then, by Theorem~\ref{thm:CBF}, we have that $\C_h$ is invariant to the closed-loop system $\dot{x} = f(x, \hat{\co}(x))$.
\end{proof}

\noindent\textbf{Learning-based synthesis.}
Next, we present the algorithm for learning the robust barrier function and the controller given the Gaussian process developed in the last section. As shown in Figure~\ref{fig:overall_sys}, we model the controller and the barrier function with two neural networks $\co(\cdot; \theta_{\co})$ and $h(\cdot; \theta_h)$, where $\theta_{\co}$ and $\theta_h$ are the parameters. To simplify the notations, we may omit $\theta_{\co}$ and $\theta_h$. The learning algorithm aims at finding the correct parameters such that $\co$ and $h$ satisfy the condition in Theorem~\ref{thm:rCBF}. To impose this condition, we define a loss function $\L_1$.
\begin{multline*}
\L_1(\theta_{h}, \theta_{\co}) = \E_{\hat{x} \sim \unif{\X}}\E_{x \sim \unif{\gp(\hat{x})}} \Bigg[\\\relu\left(- \left(\frac{\partial h}{\partial x}(x) \cdot f(x, \co(\ce(\gp(\hat{x})))) + \alpha(h(x))\right)\right)\Bigg].
\end{multline*}
The above loss function penalizes the violations of the condition of Theorem~\ref{thm:rCBF} in a probabilistic sense.
Furthermore, the 0-superlevel set of the barrier function $h$ and the user-defined safe set $\S$ should coincide. To this end, we define the following loss function.
\begin{equation*}
\L_2(\theta_h) = \E_{x \sim \unif{\X}}\left[\left(1 - \indicator_\S(x)\right) h(x) - \indicator_\S(x) h(x)\right].
\end{equation*}

In order to train the neural networks on sampled data, we transform the above loss functions into their empirical version. That is, replacing the expectations with empirical averages. To this end, we construct a data set $\D_c$ as follows. We sample $M_1$ perceived state $\hat{x}$ from $\unif{\X}$ and denote them by $\{\hat{x}_i\}_{i=1}^{M_1}$. Then, for each $\hat{x}_i$, we sample $M_2$ points $x$ from $\unif{\gp(\hat{x})}$ and denote them by $\{x_i^{j}\}_{j=1}^{M_2}$. These samples constitute the data set $\D_c := \cup_{i=1}^{M_1}\{(\hat{x}_i, x_i^{j})\}_{j=1}^{M_2}$. Then, the empirical approximation of the above loss functions is defined as follows.
\begin{multline}
\L(\theta_h, \theta_{\co}) = \frac{1}{M_1 M_2}\sum_{i=1}^{M_1} \sum_{j=1}^{M_2} \Bigg[ \\  \lambda_1 \relu\left(- \left(\frac{\partial h}{\partial x}(x_i^j) \cdot f(x_i^j, \co(\ce(\gp(\hat{x}_i))))+ \alpha(h(x_i^j))\right)\right)\\
+\lambda_2\left(\left(1 - \indicator_\S(x_i^j)\right) h(x_i^j) - \indicator_\S(x_i^j) h(x_i^j)\right)\Bigg],
\label{eq:ERM}
\end{multline}
where $\lambda_1 > 0$ and $\lambda_2 > 0$ are weights that balance two loss terms. We train the neural networks by minimizing $\L$.

\label{sec:cbf}
\subsection{Adaptive sampling for estimation and control}
\label{sec:ada}

We minimize the loss function defined in Eq.~\eqref{eq:ERM} with stochastic gradient descent. After training, the loss function might remain positive due to the large uncertainty of the Gaussian process at certain points. These points are called \textit{hard samples}.
As stated above, in order to obtain robustness, we require the barrier function condition to hold for every point in a set defined by $\gp(\hat{x})$. For a hard sample $\hat{x}$, this set is too large and satisfying this condition might be impossible. To conquer this problem, we sample more data around these hard samples to reduce uncertainty. For each hard sample $\hat{x}$, ideally we should add the sample $(\hat{x}, x)$ to $\D$ in order to reduce the uncertainty of the GP at $\hat{x}$. However, since we do not have access to $x$, we resort to evaluating the perception function at $\ce(\gp(\hat{x}))$ instead of $x$, and thus the new sample is $(\se(\ce(\gp(\hat{x}))), \ce(\gp(\hat{x})))$.
Then, we collect these samples into a set $\H$. As shown in Algorithm~\ref{alg:alg}, $\H$ is then merged with the existing samples to create a new data set for GP. The algorithm returns the control module $\cm$ if it successfully finds one, otherwise, it returns some debug information to the user such that the perception function can be improved accordingly in a separate procedure. The debug information is simply the set $\H$ in the last iteration.

\begin{algorithm}
\KwIn{Max number of iteration: $I$; Confidence $\delta$.}
\KwOut{$\co$ and $\gp$, or $\H$.}
\SetKwRepeat{Do}{do}{while}
Randomly initialize $\D = \{(\hat{x}_1, x_1), \cdots, (\hat{x}_N, x_N)\}$\;
$i \leftarrow 0$\;
\Do{$\H \neq \varnothing$ and $i < I$}{
    Compute the state estimator $\gp$ on $\D$\;
    Construct $\D_c$ and train $h$ and $\co$ on $\D_c$\;
    Collect hard samples $\H$\;
    $\D \leftarrow \D \cup \H$; $i \leftarrow i + 1$\;
}
\caption{Adaptive sampling.}
\label{alg:alg}
\end{algorithm}

By jointly improving the state estimator and the controller, we avoid learning an unnecessarily accurate state estimator at points where a rough state estimate already suffices. For example, at points that are far away from the boundary of the safe set, safe control input can be synthesized even if the state estimate is not very accurate. While at points close to the boundary, an accurate estimate becomes necessary for synthesizing a safe control input.

\label{sec:rtas}

\section{Experiments}
We evaluated the proposed approach on three dynamical systems, two simple systems with synthetic perception error functions, and one realistic simulation environment in CARLA with deep-learning-based perception modules. We begin by discussing the details of the benchmarks.

\subsection{Benchmarks and implementation details}
\noindent\textbf{Dubins vehicle.} 
The Dubins vehicle model~\cite{ dubins1957curves,sun2020learning} 
has four state variables  $p_x$, $p_y$, $\theta$, and $v$ representing the 2D position, the heading angle, and the speed. The target invariant set $\S$ 
is defined to keep the position of the vehicle inside the area enclosed with red lines as shown in Fig.~\ref{fig:density} (left). The state variables $p_x$, $p_y$, and $v$ are all perfectly observable (i.e., the perceived values equal to the actual values), but the heading angle $\theta$ is not observable. The controller only has access to a perceived heading angle $\hat{\theta} = \theta + \sin(p_x + p_y)$. Please note that such a perception function is just designed to preliminarily verify the proposed framework. As will be shown later, the proposed approach can handle much more complicated perception functions.

\noindent\textbf{Cart-pole.}
The standard cart-pole model~\cite{sun2020learning, barto1983neuronlike}
has four state variables  $p$, $v$, $\theta$, and $\omega$, which represent the position and the velocity of the cart, and the angle and angular velocity of the pole. The target invariant set is defined as $|p| < 3$ and $|\theta| < \pi / 6$. Here $p$ and $\theta$ are  perfectly observable, but  $v$ and $\omega$ are not. The controller only has access to the perceived values $\hat{v} = v + \sin(2p + 4\theta)$ and $\hat{\omega} = \omega + \cos(2p + 4\theta)$.

\noindent\textbf{Lane-keeping.}
We adopt the default vehicle model and map in the CARLA simulator~\cite{dosovitskiy2017carla} as shown in Fig.~\ref{fig:density} (right).
The key state variables of the vehicle are the lateral position $p$ and the heading angle $\theta$. The values $p=0$ define the center of the lane and $p = \pm 3.5$ define the boundary of the lane. 
The target invariant set is for the lateral position of the car to stay within the lane boundaries, i.e., $|p| < 3.5$.
The velocity of the vehicle is set to a constant.
There is a front-facing camera attached to the vehicle, and the car's controller uses the lane detection approach of~\cite{wu2021yolop} to detect the lane boundaries from the image (See  Fig.~\ref{fig:density}). Then, the lateral position and the heading angle are computed based on the detected lanes using the geometry discussed in~\cite{hsieh2021verifying}.

\vspace{1mm}
\noindent\textbf{Implementation details.}
We adopt GPytorch~\cite{gardner2018gpytorch} as the programming framework for GP. The confidence level is set to $\delta = 0.95$. For all experiments in this section, we model the controller and the barrier function with two $3$-layer neural networks, of which the hidden layer contains $128$ neurons. The $K_\infty$ function $\alpha$ is designed as $\alpha(x) = 0.1 x$. The loss weights in Eq.\eqref{eq:ERM} are $\lambda_1 = 0.01$ and $\lambda_2 = 1$. To construct the data set $\D_c$ for learning the controller and barrier function, we set $M_1 = 10000$ and $M_2 = 32$. We train the neural networks for $30$ epochs with a learning rate $0.1$. The maximum number of iterations is set to $I = 6$.

\subsection{Experimental results}
\label{ssection:results}
As a performance index for the learned controller, we report the \textit{unsafe ratio}, which empirically measures the fraction of finite-time ($10$ seconds in our experiments) trajectories that exit the invariant set $\S$ from a critical set of initial conditions. The choice of the initial condition is important because the unsafe ratio will be favorable for an initial set that is far from the  boundary of $\S$.
In our experiments, the critical initial set is chosen such that the trajectories exit $\S$ in $1$ second if the control input $u$ is set to zero.

In Fig.~\ref{fig:sample_efficency}, we show how the unsafe ratio varies with the number of samples in the data set $\D$. We compared the proposed adaptive sampling approach with uniform sampling, where instead of enlarging the data set $\D$ with hard samples we use data points that are uniformly sampled from the state space, and the baseline, where GP is disabled and the state estimator is set to $\gp(\hat{x}) = \{\hat{x}\}$. There are several observations from the experiments.
(1) Both methods find controllers that are robust to the perception error.
(2) Adaptive sampling leads to a lower unsafe ratio than uniform sampling with the same number of samples in $\D$, which empirically confirms our intuition  that the proposed controller design approach with adaptive sampling is more sample efficient.
To further illustrate the benefit of adaptive sampling, we contrast the data sets $\D$ generated by adaptive and uniform sampling in Fig.~\ref{fig:density} (left). 
As can be seen from the figure, adaptive sampling zooms in on relevant areas of the state space, because at those areas, a more accurate state estimator is needed, while uniform sampling evenly distributes its sampling budget.
Furthermore, a video demonstration of the results on the lane-keeping benchmark can be found in the supplementary material.

\begin{figure}
    \centering
    \includegraphics[width=\linewidth]{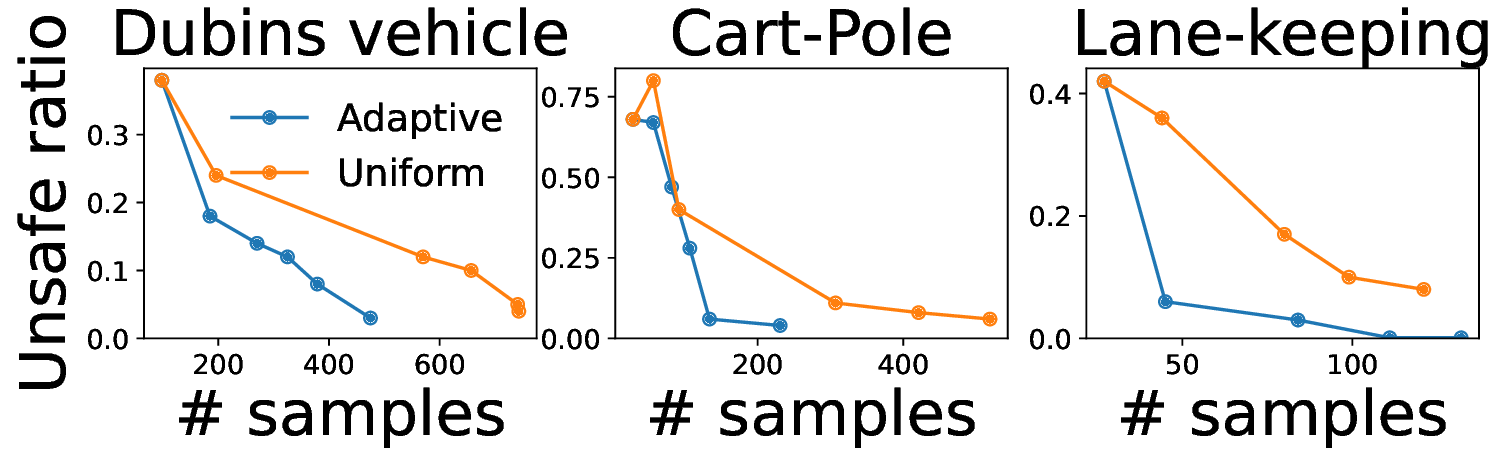}
    \caption{\small Comparison of sample efficiency of learning certifiable controller. While both uniform and adaptive sampling approaches learn safety preserving controllers, the latter is significantly more sample efficient. The baseline for the three benchmarks are $0.58$, $0.85$, and $0.61$ respectively.}
    \label{fig:sample_efficency}
    \vspace{-3mm}
\end{figure}

\begin{figure}
    \centering
    \includegraphics[width=.28\linewidth]{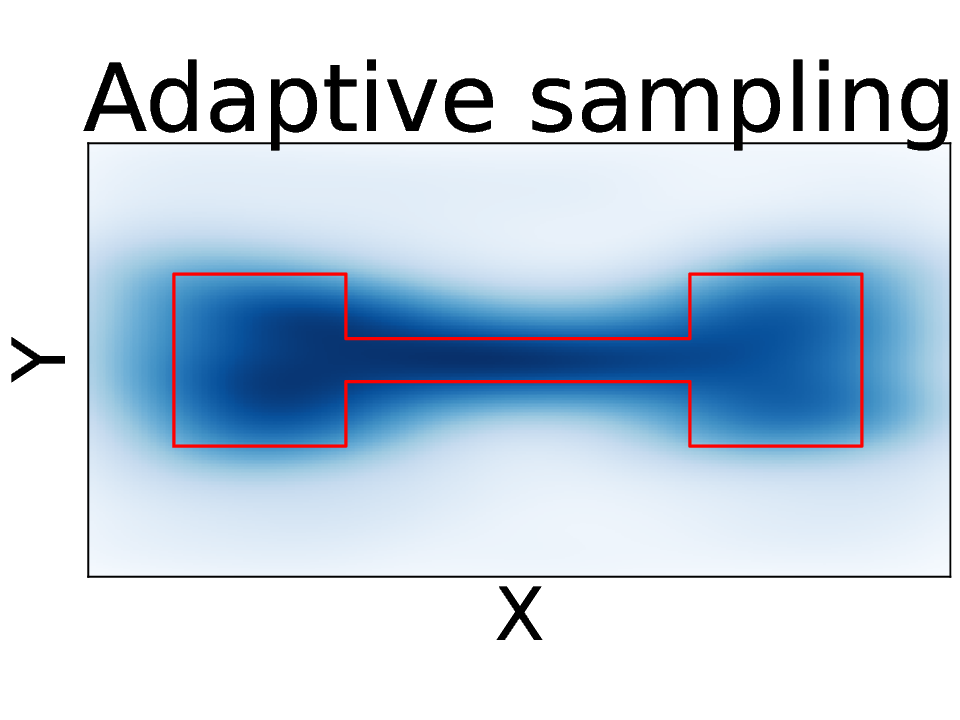}
    \includegraphics[width=.28\linewidth]{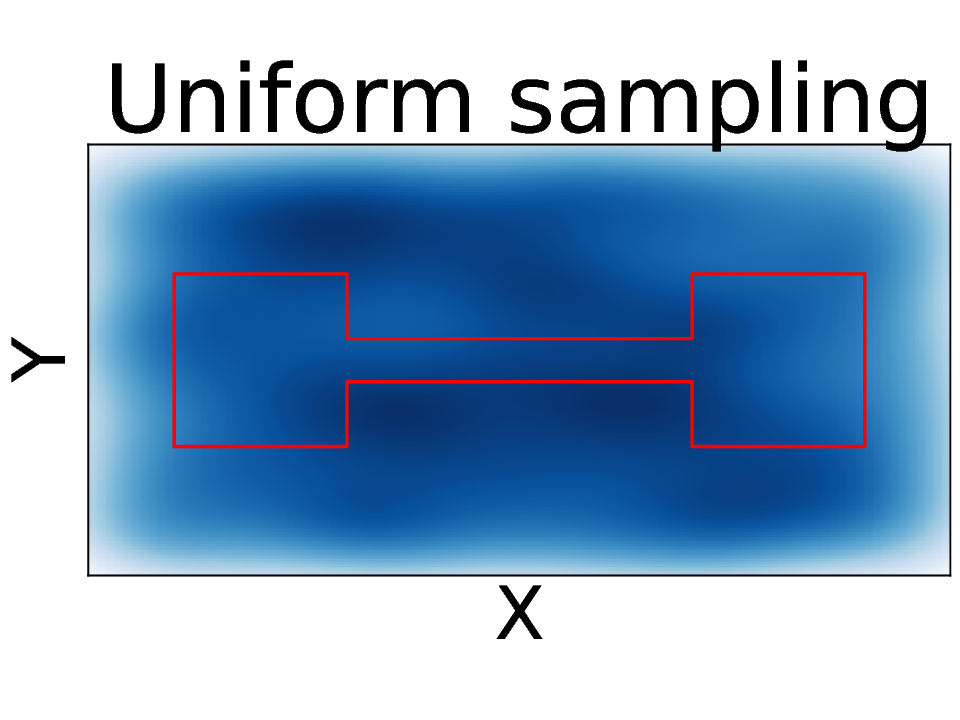}
    \includegraphics[width=.2\linewidth, trim={30mm 10mm 30mm 10mm}, clip]{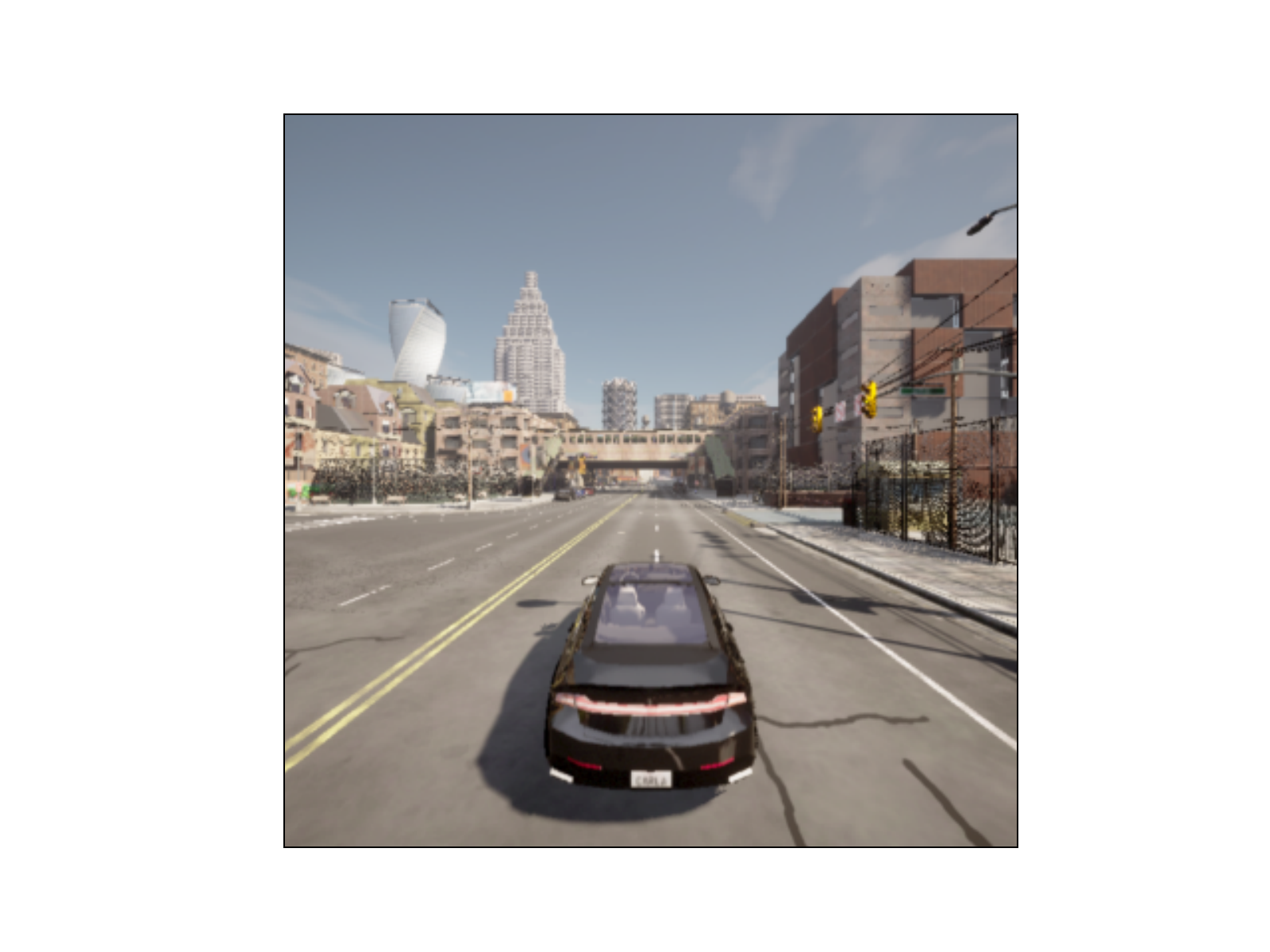}
    \includegraphics[width=.2\linewidth, trim={30mm 10mm 30mm 10mm}, clip]{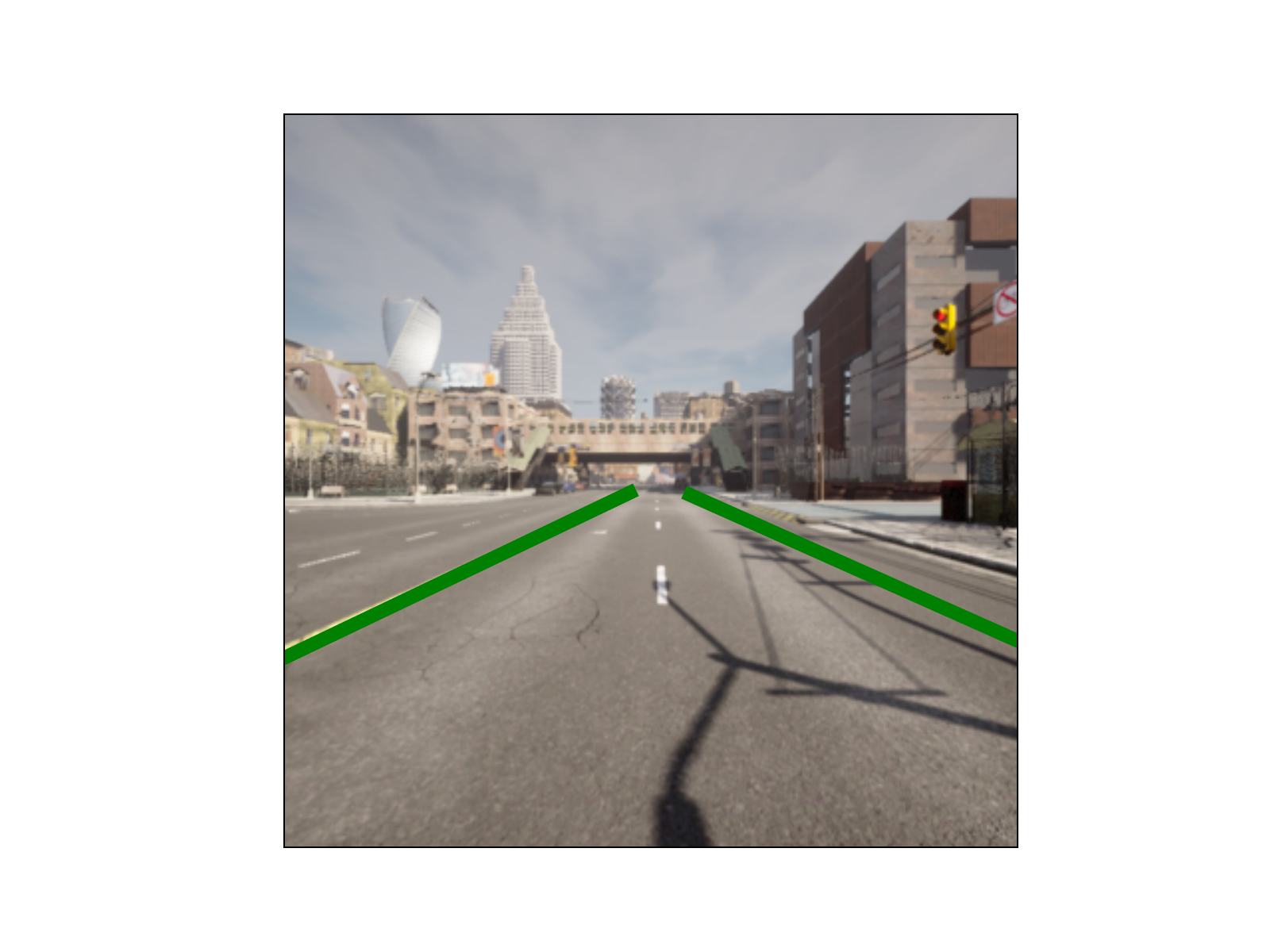}
    \caption{\small Distribution of samples in the data set $\D$ constructed by two sampling approaches on Dubins vehicle (left). Darker color means higher density. The simulated lane-keeping scenario in CARLA and the detected lanes (right).}
    \label{fig:density}
\vspace{-5mm}
\end{figure}

\section{Limitations and future work}
One should be careful with interpreting the theoretical guarantee of the proposed approach.
The confidence $\delta$ characterizes the probability of each standalone state falling into the high-confidence set. Further theoretical treatment is required to boost such a guarantee to one on the safety of trajectories.
Furthermore, in order to handle larger data sets, approximate GP such as~\cite{hensman2015scalable} has to be used.

\bibliographystyle{IEEEtran}
\bibliography{IEEEabrv, references}

\end{document}